\documentclass{article} 
\usepackage{iclr2024_conference,times}
\usepackage{graphicx}

\usepackage{amsmath,amsfonts,bm}

\def\eqref#1{equation~\ref{#1}}

\def\1{\bm{1}}

\def\eps{{\epsilon}}

\def\ve{{\bm{e}}}

\DeclareMathAlphabet{\mathsfit}{\encodingdefault}{\sfdefault}{m}{sl}
\SetMathAlphabet{\mathsfit}{bold}{\encodingdefault}{\sfdefault}{bx}{n}

\newcommand{\R}{\mathbb{R}}

\DeclareMathOperator*{\argmin}{arg\,min}

\usepackage{hyperref}
\usepackage{url}

\usepackage{amsmath}
\usepackage{amssymb}
\usepackage{mathtools}
\usepackage{amsthm}
\usepackage{xcolor}
\usepackage{booktabs}
\usepackage[bbgreekl]{mathbbol}
\newif
\ifdraft
\draftfalse
\def\ale#1{\ifdraft\textcolor{magenta}{#1}\else{#1}\fi}
\DeclareSymbolFontAlphabet{\mathbbl}{bbold}
\ifdraft
\usepackage{lipsum}
\usepackage[notref,notcite,color]{showkeys}
\usepackage[cam,width=10in,height=12in,center]{crop}
\definecolor{refkey}{rgb}{0,.6,0}
\definecolor{labelkey}{cmyk}{0, 1, 0, 0}
\fi

\def\R{\mathbbl{R}}
\def\eps{\varepsilon}
\def\ch{\mathop{\rm ch}\nolimits}
\def\pa{\mathop{\rm pa}\nolimits}

\def\proj{\mathop{\textsc{pa}}\nolimits}
\def\inweights{\mathop{\textsc{in}}\nolimits}
\def\allweights{{\boldsymbol{w}}}
\def\v{{\boldsymbol{v}}}
\def\u{{\boldsymbol{u}}}
\def\nuu{{\boldsymbol{\nu}}}
\def\a{{\boldsymbol{a}}}

\definecolor{dgreen}{rgb}{0, 0.5, 0}
\definecolor{orange}{rgb}{1, 0.37, 0.12}
\def\input{u}

\theoremstyle{plain}
\newtheorem{proposition}{Propostion}
\newtheorem{lemma}{Lemma}
\newtheorem{theorem}{Theorem}
\newtheorem{corollary}{Corollary}
\newtheorem{conjecture}{Conjecture}

\theoremstyle{definition}
\newtheorem{definition}{Definition}
\newtheorem{example}{Example}

\theoremstyle{remark}
\newtheorem{remark}{Remark}

\def\state{y} 
\def\control{\alpha} 
\def\controlbis{\beta} 
\def\time{t}
\def\statevar{\xi} 
\def\costatevar{\eta} 
\def\controlvar{a} 
\def\timevar{s}

\def\f{f}

\def\cost{C}
\def\loss{\Lambda}
\def\controlspace{\mathcal{A}}
\def\valuef{v}

\def\costatevar{\rho}
\def\statedim{n} 
\def\controldim{N} 

\title{
Nature-Inspired Local Propagation}

\author{Alessandro Betti\\
Inria, Lab I3S, MAASAI, Universit\`e C\^ote d'Azur, Nice, France\\
\texttt{alessandro.itteb@gmail.com} \\
\And
Marco Gori\\
DIISM, University of Siena, Siena, Italy\\
\texttt{marco@diism.unisi.it} \\
}

\iclrfinalcopy 
\begin{document}

\maketitle

\begin{abstract}
The spectacular results achieved in machine learning, including the
recent advances in generative AI, rely on large data collections. On the
opposite, intelligent processes in nature arises without the need for such
collections, but simply by on-line processing of the environmental information.
In particular, natural learning processes rely on mechanisms where data
representation and learning are intertwined in such a way to respect
spatiotemporal locality. This paper shows that such a feature arises from
a pre-algorithmic view of learning that is inspired by related studies
in Theoretical Physics. We show that the algorithmic
interpretation of the derived ``laws of learning'', which takes the structure
of Hamiltonian equations, reduces to Backpropagation when the speed of
propagation goes to infinity. This opens the doors to 
machine learning studies based on full on-line information processing
that are based the replacement of Backpropagation with 
the proposed spatiotemporal local algorithm. 
\end{abstract}

\section{Introduction}
By and large, the spectacular results of Machine Learning in nearly any application
domain are strongly relying on large  data collections along with associated 
professional skills. Interestingly, the successful artificial schemes that 
we have been experimenting under this framework are far away from the solutions
that Biology seems to have discovered.
We have recently seen a remarkable effort in the scientific community of 
explore biologically inspired models (e.g. see~\cite{scellier2017equilibrium}, \cite{kendall2021gradient, scellier2021deep, laborieux2022holomorphic}) where
the crucial role of temporal information processing it is clearly identified. 

While this paper is related to those investigations, it is based on 
more strict assumptions on environmental interactions that might stimulate 
efforts towards  a more radical transformation of machine 
learning with emphasis on temporal processing.  
In particular, we assume that learning and inference develop jointly under a 
nature based protocol of environmental interactions and then we suggest developing 
computational learning schemes regardless of biological solutions.
Basically,  the agent is not given the privilege of recording 
the temporal stream, but only to represent it properly by appropriate abstraction
mechanisms. While the agent can obviously use its internal memory for 
storing those representations, we assume that it cannot access data collection.
The agent can only use buffers of limited dimensions for storing the processed
information. From a cognitive point of view, those small buffers  allows the agent to 
process the acquired information backward in time by implementing a 
sort of focus of attention. 

We propose a pre-algorithmic approach which derives
from the formulation of learning as an Optimal Control problem and 
propose an approach to its solution that is also inspired by 
principles of Theoretical Physics.
We formulate the continuous learning problem to emphasize how 
optimization theory brings out solutions based on differential equations 
that recall similar laws in nature. The discrete counterpart, which is 
more similar to recurrent neural network algorithms that can be found in the 
literature, can promptly be derived by numerical solutions.  

Interestingly,  we show that the on-line computation described in the paper
yields spatiotemporal locality, thus addressing also the longstanding debate on
Backpropation biological plausibility. 
Finally, the paper shows that the conquest of locality opens
up a fundamental problem, namely that of approximating the solution of
Hamilton's equations with boundary conditions using only initial conditions.
A few insights on the solution of this problem are given for the task of
tracking in optimal control, which opens the doors of a massive investigation
of the proposed approach.



\section{Recurrent Neural Networks and spatiotemporal locality} \label{sec:model}
We put ourselves in the genral case where the computational 
model that we are considering is based on a digraph 
$D=(V,A)$ where $V=\{1,2,\dots, n\}$ is the set of vertices and 
and $A$ is the set of directed arches 
that defines the structure of the graph. Let $\ch(i)$ denote
the set of vertices that are childrens of vertex $i$ and with
$\pa(i)$ the set of vertices that are parents of vertex $i$
for any given $i\in V$.
More precisely we 
are interested in the computation of neuron outputs over a 
temporal horizon $[0,T]$. Formally this amounts to determine to
assign to each vertex $i\in V$ a trajectory $x_i$ that 
is computed parametrically in terms of the other neuron outputs 
and in terms of an environmental information,
mathematically represented by a trajectory\footnote{In the reminder of the paper we will try
whenever possible to formally introduce functions by clearly stating
domain a co-domain.
In particular whenever the function acts on a product space we will try
to use a consistent notation for the elements in the various
sets that define the input so that we can re-use such notation to denote
the partial derivative of such function. For instance let us suppose that
$f\colon A\times B\to\R$ is a function that maps $(a,b)\mapsto f(a,b)$
for all $a\in A$ and $b\in B$. Then 
we will denote with $f_a$ \emph{the function}
that represents the partial derivative
of $f$ with respect to its first argument, with $f_b$ the
partial derivative of $f$ with respect to its second
argument \emph{as a function} and so on. We will instead denote, for
instance, with $f_a(x,y)$ the element of $\R$ that represent the value
of $f_a$ on the point $(x,y)\in A\times B$.}
$u\colon
[0,+\infty)\to\R^d$. We will assume that
the output of the first $d$ neurons (i.e the value of
$x_i$ for $i=1,\dots, d$) matches the value of the components of the input:
$x_i(t) = u_i(t)$ for $i=1,\dots d$ and $\forall t\in[0,T]$.
In order to consistently interpret the first $d$ neurons as input we
require two additional property of the graph structure:
\begin{align}
&\pa(i)=\emptyset\quad \forall i=1,\dots,d;
\label{eq:input_do_not_have_parents}\\
&\pa(\{d+1,\dots, n\})\supset \{1,\dots, d\}
\label{eq:all_inputs_are_connected}.
\end{align}
Here \eqref{eq:input_do_not_have_parents} says that an input neuron
do not have parents, and it also implies that no self loops are allowed
for the input neurons.
On the other hand \eqref{eq:all_inputs_are_connected} means that
all input neurons are connected to at least one other neuron
amongst $\{d+1,\dots, n\}$.

We will denote with $x(t)$ (without a subscript)
the ordered list of all the output of the neurons at time $t$
except for the input neurons,
$x(t):=(x_{d+1}(t),\dots, x_n(t))$, and with this definition we can 
represent $x(t)$ for any $t\in[0,T]$ as a vector in the euclidean 
space $\R^{n-d}$. This vector is usually called the \emph{state} of the
network since its knowledge gives you the precise value of each
neuron in the net.
The parameters of the model are instead associated to
the arcs of the graph via the map $(j,i)\in A\mapsto w_{ij}$
where $w_{ij}$ assumes values on $\R$. We will denote with
$w_{i*}(t)\in\R^{|\pa(i)|}$ the vector composed of all the weights
corresponding to arches of the form $(j,i)$. If we let $N:=\sum_{i=1}^n
|\pa(i)|$ the total number of weights of the model we also define
$\R^N\ni\allweights(t):=(w_{1*}(t),
\dots, w_{n*}(t))$ the concatenation of all the weights of the network.
Finally we will assume that the output of the model is computed
in terms of a subset of the neurons. More precisely we will assume that
a vector of $m$ indices $(i_1,\dots, i_m)$ with $i_k\in\{d+1,\dots, n\}$
and at each temporal instant the output of the net
is a function $\pi\colon\R^m\to\R^h$ of $(x_{i_1},\dots,x_{i_m})$, that
is $\pi(x_{i_1},\dots,x_{i_m})$ is the output of our model.
For future convenience we will denote  $O=\{
i_1,\dots, i_m\}$.

\paragraph{Temporal locality and causality}
In general we are interested in computational schemes which are
both local in time and causal. Let us assume that we are working at some fixed
temporal resolution $\tau$, meaning that we can define a partition
of the half line $(0,+\infty)$, $\mathcal{P}:=\{
0=t_\tau^0<t_\tau^1<\dots<t_\tau^n<\dots\}$ with
$t_\tau^n=t^{n-1}_\tau+\tau$, then
the input signal becomes a sequence of vectors $(U^n_\tau)_{n=0}^{+\infty}$
with $U^n_\tau:= u(t^n_\tau)$ and the neural outputs and parameters can
be regarded as an approximation of the trajecotries $x$ and $\allweights$:
$X^n_\tau\approx x(t^n_\tau)$ and $W^n_\tau\approx \allweights(t^n_\tau)$
$n=1,\dots \lfloor T/\tau\rfloor$. A local computational 
rule for the neural outputs means that $X^n_\tau$ is a function of
$X_\tau^{n-l},\dots, X_\tau^{n}, \dots, X_\tau^{n+l}$,
$W_\tau^{n-l},\dots, W_\tau^{n}, \dots, W_\tau^{n+l}$
and $t_\tau^{n-l},\dots, t_\tau^{n}, \dots, t_\tau^{n+l}$, where
$l\ll T/\tau$ can be thought as the order of locality.
If we assume that $l\equiv 1$ (first order method) which means that
\begin{equation}\label{eq:locality-disc}
X^n_\tau=F(X^{n-1}_\tau, X^n_\tau, X^{n+1}_\tau,
W^{n-1}_\tau, W^n_\tau, W^{n+1}_\tau,
t^{n-1}_\tau, t^n_\tau, t^{n+1}_\tau).\end{equation}
Causality instead express the fact that only past
information can influence the current state of the variables meaning
that actually \eqref{eq:locality-disc} should be replaced by
$X^n_\tau=F(X^{n-1}_\tau, W^{n-1}_\tau, t^{n-1}_\tau)$.
Going back to the continuous description 
this equation can be interpreted as a
discretization of
\begin{equation}\label{eq:local+causality-cont}
\dot x = f(x,\allweights,t),
\end{equation}
with initial conditions.

\paragraph{Spatial locality}
Furthermore we assume that such computational scheme is local in 
time and make use only on spatially local (with respect to the structure
of the graph) quantities more precisely has the following structure 
\begin{equation}\label{eq:spatial_locality_model}
\begin{cases}
x_i(t) = u_i(t) &\hbox{for $i=1,\dots d$}\quad\hbox{and}
\quad\hbox{$\forall t\in[0,T]$};\\
c_i^{-1}\dot x_i(t)=\Phi^i(x_i(t),\proj^i(x(t)), \inweights^i(
\allweights(t)))
&\hbox{for $i=d+1,\dots, n$}\quad\hbox{and}
\quad\hbox{$\forall t\in[0,T]$},
\end{cases}
\end{equation}
Where $c_i>0$ for all $i=d+1,\dots, n$ set the velocity constant
of that controls the updates of the $i$-th neuron,
$\Phi^i\colon\R\times\R^{|\pa(i)|}\times\R^{|\pa(i)|}\to\R$
for all $i=d+1,\dots, n$ performs the mapping
$(r,\alpha,\beta)\mapsto\Phi^i(r,\alpha,\beta)$ for all
$r\in\R$, $\alpha,\beta\in\R^{|\pa(i)|}$,
$\proj^i\colon\R^{n-d}\to\R^{|\pa(i)|}$ project the vector
$\xi\in\R^{n-d}\mapsto\proj^i(\xi)$
on the subspace generated by neurons which are in $\pa(i)$ and
$ \inweights^i\colon\R^N\to\R^{|\pa(i)|}$ maps the any vector
$\u\in\R^N\mapsto \inweights^i(\u)$
 onto the space spanned by only the weights associated to
arcs that points to neuron $i$.
The assumptions summarized above describe the basic properties of a
RNN or, as sometimes is referred to when dealing with a continuous time
computation, a Continuous Time RNN \cite{sompolinsky1988chaos}.
The typical form of function $\Phi_i$, is the following
\begin{equation}\label{eq:ctrnn}
\Phi^i(r,\alpha,\beta)=-r +\sigma(\beta\cdot\alpha),\quad
\forall r\in\R \quad\hbox{and}\quad \forall \alpha,\beta\in\R^{|\pa(i)|}.
\end{equation}
where in this case $\,\cdot\,$ is the standard scalar product on
$\R^{|\pa(i)|}$ and $\sigma\colon\R\to\R$ is a nonlinear
bounded smooth function
(usually a sigmoid-like activation function). Under this assumption
the state equation in \eqref{eq:spatial_locality_model} becomes 
\begin{equation}
c_i^{-1}\dot x_i(t)=-x_i(t) +\sigma(\inweights^i(\allweights(t))
\cdot \proj^i(x(t)))\equiv
-x_i(t) +\sigma\Bigl(\sum_{j\in\pa(i)} w_{ij}x_j(t)\Bigr),
\label{RecNeuralNeteq}
\end{equation}
which is indeed the classical neural computation. Here we sketch a result
on the Bounded Input Bounded Output (BIBO) stability of this class of
recurrent neural  network which is also important for the learning process
that will be described later.

\begin{proposition}\label{prop:rnn-stability}
    The recurrent neural network defined by ODE~(\ref{RecNeuralNeteq}) 
    is  (BIBO) stable.
\end{proposition}
\begin{proof}
See Appendix~\ref{rnn-stability}
\end{proof}

\section{Learning as a Variational Problem}\label{sec:optimization}
In the computational model described in Section~\ref{sec:model}, once
the graph $D$ and an input $u$ is assigned, the dynamics of the
model is determined solely by the functions that describes the changes of
the weights over time.
Inspired by the Cognitive Action Principle \cite{betti2019cognitive}
that formulate learning for FNN in terms of a variational problem,
we  claim that in an online setting the laws of learning
for recurrent architectures can also be
characterized by minimality of a class of functional.
In what follows we will then consider variational problems for a functional
of the form
\begin{equation}\label{eq:general-functional}
F(\allweights)=\int_0^T
\biggl[\frac{mc}{2}|\dot\allweights|^2 +c\ell(\allweights(t),
x(t;\allweights),t)\biggr]\phi(t)\, dt,
\end{equation}
where $x(\cdot,\allweights)$ is the solution of
\eqref{eq:local+causality-cont} with fixed initial
conditions\footnote{We do not explicitly indicate the dependence on
the initial condition to avoid cumbersome notation.},
$\phi\colon[0,T]\to\R$ is a strictly positive smooth function that
weights the integrand, $m>0$,
$\ell\colon \R^n\times\R^N\times[0,T]\to\overline \R_+$ is a positive
function  and
finally $c:=\sum_{i=d+1}^n c_i/(n-d)$. We discuss under which conditions
the stationarity conditions of this class
of functional can be made temporally and spatially local and how they
can be interpreted as learning rules.


\subsection{Optimal Control Approach}
The problem of minimizing the functional in
\eqref{eq:general-functional}  can be solved by making use
of the formalism of Optimal Control. The first step is to put this problem
in the canonical form by introducing an additional control variable
as follow
\begin{equation}\label{eq:funtional-G}
G(\v)=\int_0^T
\biggl[\frac{mc}{2}|\v|^2 +c\ell(\allweights(t;\v),
x(t;\v),t)\biggr]\phi(t)\, dt,
\end{equation}
where $\allweights(t;\v)$ and $x(t;\v)$ solve
\begin{equation}
\label{eq:funtional-G-constraints}
\dot x(t) = f(x(t),\allweights(t),t),\quad \hbox{and}\quad
\dot \allweights(t) = \v(t).
\end{equation}
Then the minimality conditions can be expressed in terms of the
Hamiltonian function (see Appendix~\ref{appendix:control}):
\begin{equation}\label{eq:hamiltonian-general}
H(\xi,\u,p,q,t)=
-\frac{1}{\phi(t)}\frac{q^2}{2mc}+ c\ell(\u,
\xi,t)\phi(t) + p\cdot f(\xi,\u,t),
\end{equation}
via the following general result.
\begin{theorem}[Hamilton equations]
Let $H$ be as in \eqref{eq:hamiltonian-general} and
assume that $x(0)=x^0$ and $\allweights(0)=\allweights^0$
are given. Then a minimum of the functional
in \eqref{eq:funtional-G} satisfies the Hamilton
equations:
\begin{equation}\label{eq:hamilton-general}
\begin{cases}
\dot x(t)= f(x(t),\allweights(t),t)\\
\dot\allweights(t)=-p_\allweights(t)/(mc\phi(t))\\
\dot p_x(t)=-p_x(t)\cdot f_\xi(x(t),\allweights(t),t)-
 c\ell_\xi(\allweights(t),x(t),t)\phi(t)\\
\dot p_\allweights(t)=-p_x(t)\cdot f_\u(x(t),\allweights(t),t)-
 c\ell_\u(\allweights(t),x(t),t)\phi(t)
\end{cases}
\end{equation}
together with the boundary conditions
\begin{equation}\label{eq:boundary}
p_x(T)=p_\allweights(T)=0.
\end{equation}
\end{theorem}
\begin{proof}
See Appendix~\ref{appendix:control}
\end{proof}

\subsection{Recovering spatio-temporal locality}\label{sec:stlocality}
Starting from the general expressions for the stationarity conditions
expressed by \eqref{eq:hamilton-general} and \eqref{eq:boundary}, 
we will now discuss how the temporal and spatial locality assumptions
that we made on our computational model in Section~\ref{sec:model}
leads to spatial and temporal locality of the update rules of the
parameters $\allweights$.

\paragraph{Temporal Locality}
The local structure of \eqref{eq:funtional-G-constraints}, that comes from
the locality of the computational model that we discussed in
Section~\ref{sec:model} guarantees the locality of Hamilton's
equations~\ref{eq:hamilton-general}. However the functional
in \eqref{eq:funtional-G} has a global nature (it is an integral over the
whole temporal interval) and the differential term $m |\v|^2/2$ links
the value of the parameters across near temporal instant giving rise to
boundary conditions in \eqref{eq:boundary}. This also means that, strictly
speaking \eqref{eq:hamilton-general} and \eqref{eq:boundary}
overall define a problem that is non-local in time.
We will devote the entire Section~\ref{sec:boundary} to discuss this
central issue.

\paragraph{Spatial Locality}
The spatial locality of \eqref{eq:hamilton-general} directly comes from
the specific form of the dynamical system in \eqref{eq:spatial_locality_model}
and from a set of assumptions on the form of the term $\ell$. In particular
we have the following result:

\begin{theorem}
Let $\ell(\u,\xi,s)=kV(\u,s)+L(\xi,s)$ for every $(\u,\xi,s)\in
\R^N\times\R^{n-d}\times[0,T]$, where
 $V\colon\R^N\times[0,T]\to\overline\R_+$
is a regularization term on the weights\footnote{a typical choice for this function
could be $V(\u,s)=|\u|^2/2$ with $k>0$} and 
$L\colon \R^{n-d}\times[0,T]\to\overline\R_+$ depends only on
the subset of neurons from which we assume the output of the model
is computed, that is
$L_{\xi_i}(\xi,s) = L_{\xi_i}(\xi,s)1_O(i)$, where
$1_O$ is the indicator function of the set of the output neurons.
Let $\Phi^i$ be as in \eqref{eq:ctrnn} for all $i=d+1,\dots, n$, then 
the generic Hamilton's equations described in \eqref{eq:hamilton-general}
become
\begin{equation}\label{eq:local_HJ-second-form}
\begin{cases}
c_i^{-1}\dot x_i= -x_i +\sigma\Bigl(\sum_{j\in\pa(i)} w_{ij}x_j\Bigr)\\
\dot w_{ij}=-p^{ij}_\allweights/(mc\phi)\\
\dot p^i_x=c_i p_x^i-\sum_{k\in\ch(i)} c_k 
\sigma'\Bigl(\sum_{j\in\pa(k)} w_{kj}x_j\Bigr)p_x^k w_{ki}
- c L_{\xi_i}(x,t)\phi\\
\dot p^{ij}_\allweights(t)=
- c_i p^i_x \sigma'\Bigl(\sum_{m\in\pa(i)} w_{im}x_m\Bigr)x_j
- c k V_{\u_{ij}}(\allweights,t)\phi
\end{cases}
\end{equation}
\label{HL-ff-eqs}
\end{theorem}
\begin{proof}
See Appendix~\ref{append:main-theo}
\end{proof}

\begin{remark}
Notice \eqref{eq:local_HJ-second-form}
directly inherit the spatially local structure
from the assumption in \eqref{eq:spatial_locality_model}.
\end{remark}

Theorem~\ref{HL-ff-eqs} other than giving us a spatio-temporal show that the
computation of the $x$ costates has a very distinctive and familiar structure:
for each neuron the values of $p^i_x$ are computed using quantities defined
on chilren's nodes as it happens for the computations of the gradients in
the Backpropagation algorithm for a FNN. In order to better understand
the structure of \eqref{eq:local_HJ-second-form} let us
define an appropriately normalized costate
\begin{equation}\label{eq:rescaled-costates}
\lambda^i_x(t):= \frac{\sigma'(a_i(t))}{\phi(t)}p^i_x(t),\quad
\hbox{with}\quad 
a_i(t)= \sum_{m\in\pa(i)} w_{im}x_m \quad \forall i=d+1,\dots, n,
\end{equation}
where we have introduced the notation $a_i(t)$ to stand for the
activation of neuron $i$.\footnote{We have
avoided to introduce the notation till now because
we believe that it is worth writing \eqref{eq:local_HJ-second-form}
with the explicit dependence on the variable $w$ and $x$ at least one
to better appreciate its structure.}
With these definitions we are ready to state the following result

\begin{proposition}\label{prop:second-order-dyn}
The differential system in \eqref{eq:local_HJ-second-form} is
equivalent to the following system of ODE of mixed orders:
\begin{equation}\label{eq:second-order-eq}
\left\{\begin{aligned}
&c_i^{-1}\dot x_i= -x_i +\sigma(a_i);\\
&\ddot w_{ij}=-\frac{\dot \phi}{\phi} \dot w_{ij}
+\frac{c_i}{mc}\lambda^i_x x_j
+\frac{k}{m} V_{\u_{ij}}(\allweights,t);\\
&\dot \lambda^i_x=
\biggl[-\frac{\dot\phi}{\phi}+\frac{d}{dt}\log(\sigma'(a_i))
+c_i\biggr]\lambda^i_x -\sigma'(a_i)\sum_{k\in\ch(i)}c_k\lambda^k_x w_{ki}
- c L_{\xi_i}(x,t)\sigma'(a_i),
\end{aligned}\right.
\end{equation}
where $\lambda^i_x$ is defined as in \eqref{eq:rescaled-costates}.
\end{proposition}
\begin{proof}
See Appendix~\ref{append:gradient}
\end{proof}

This in an interesting result especially since via the following
corollary gives a direct link between the rescaled costates $\lambda_x$
and the delta error of Backprop:
\begin{corollary}[Reduction to Backprop]\label{cor:delta-error}
Let $c_i$ be the same for all $i=1,\dots, n$ so that now $c_i=c$, 
then the formal limit of the $\dot\lambda_x$ equation in the system
\ref{eq:second-order-eq} as $c\to\infty$ is
\begin{equation}
\lambda^i_x= \sigma'(a_i)\sum_{k\in\ch(i)}\lambda^k_x w_{ki}
+ L_{\xi_i}(x,t)\sigma'(a_i).
\end{equation}
\end{corollary}
\begin{proof}
The proof comes immediately from \eqref{eq:second-order-eq}.
\end{proof}

Notice that the equation for the $\lambda$ is exactly the update equation for
delta errors in backpropagation: when $i$ is an output neuron its value
is directly given by the gradient of the error, otherwise it is express as a
sum on its childrens (see~\cite{gori2023machine}).

\section{From boundary to Cauchy's conditions}\label{sec:boundary}
While discussing temporal locality in Section~\ref{sec:optimization}, we came
across the problem of the left boundary conditions on the costate variables.
We already noticed that these constraints spoil the locality of the
differential equations that describe the minimality conditions of the
variational problem at hand. In general, this prevents us from computing
such solutions with a forward/causal scheme.

The following examples should suffice to explain that, in general, this is a
crucial issue and should serve as motivation for the further investigation we
propose in the present section.

\begin{example}\label{ex:null-px}\label{ex:forward-blow-up}
Consider a case in which 
$\ell(\u,\xi,s)\equiv V(\u,s)$, i.e. we want to study the minimization  
problem for $\int_0^T(m|\v(t)|^2/2 + V(\allweights(t;\v),t))\phi(t)dt$
under the constraint $\dot\allweights=\v$.
Then  the dynamical equation $\dot x(t)=f(x(t),\allweights(t))$
does not represent a constraint on variational problems for
functional in \eqref{eq:funtional-G}. If we look at the Hamilton equation
for $\dot p_x$ in \eqref{eq:hamilton-general} this reduces to
$\dot p_x= -p_x\cdot f_\u$. We would however expect
$p_x(t)\equiv 0$ for all $t\in[0,T]$. Indeed this is the solution that
we would find if we pair $\dot p_x= -p_x\cdot f_\u$ with its boundary
condition $p_x(T)=0$ in \eqref{eq:boundary}. Notice that however in general
without this condition a random Cauchy initialization of this equation
 would not give null solution for the $x$ costate.
 Now assume  that $\phi=\exp(\theta t)$ with
$\theta>0$,  and $m=1$. Assume,
furthermore\footnote{The same argument that we give in this example works
for a larger class of cohercive potentials $V$.}
that $V(\u,s)=|\u|^2/2$. The functional $\int_0^T (|\dot w|/2+ |w|^2/2) e^{\theta t}dt$
defined over the functional space\footnote{These are called Sobolev spaces,
for more details see~\cite{brezis2011functional}.}
$H^1([0,T];\R^N)$ is cohercive and lower-semicontinuous, and hence admits
a minimum (see~\cite{giaquinta1995calculus}). Furthermore one can prove (see~\cite{liero2013new})
that such minimum is actually $C^\infty([0,T];\R^N)$.
This allows us to describe such minumum with the Hamilton equations described
in \eqref{eq:hamilton-general}. In particular as we already commented
the relevant equations are only that
for $\dot\allweights$ and $\dot p_\allweights$ that is 
$\dot\allweights(t)=-p_\allweights(t)e^{-\theta t}$ and 
$\dot p_\allweights(t)=-\allweights e^{\theta t}$
with $p_\allweights(T)=0$.
This first order system of equations is equivalent to the second
order differential equation $\ddot\allweights(t)+\theta\dot\allweights(t)
-\allweights(t)=0$. Each component of this second order system will,
in general have an unstable behaviour since one of the eigenvalues is
always real and positive. This is a strong indication that when solving
Hamilton's equations with an initial condition on $p_\allweights$ we will
end up with a solution that is far from the minimum.
\end{example}

In the next subsection, we will analyze this issue in more detail and present
some alternative ideas that can be used to leverage Hamilton's equations for
finding causal online solutions.

\subsection{Time Reversal of the Costate}\label{sec:flip-pdot}
In Example~\ref{ex:forward-blow-up} we discussed how
the forward solution of Hamilton's \eqref{eq:hamilton-general}
with initial conditions both on the state and on the costate
in general cannot be related to any form of minimality of
the cost function in \eqref{eq:funtional-G} and this has
to do with the fact that the proper minima are characterized
also by left boundary conditions~\ref{eq:boundary}.
The final conditions on $\dot p_x$ and $\dot p_\allweights$
suggest that the costate equations should be solved backward in time.
Starting form the final temporal horizon and going backward in time
is also the idea behind dynamic programming, which is
of the main ideas at the very core of optimal control theory. 

Autonomous systems of ODE with terminal
boundary conditions can be solved ``backwards'' by
time reversal operation $t\to -t$ and transforming terminal into
initial conditions. More precisely the following classical result holds:
\begin{proposition}\label{prop:time-reversal}
Let $\dot y(s) = \varphi(y(t))$ be  a system
of ODEs on  $[0,T]$ with terminal conditions $y(T)=y^T$
and let $\rho$ be the time reversal transformation
maps $t\mapsto s=T-t$, then $\hat y(s):= y(\rho^{-1}(s))= y(t)$
satisfies $\dot{\hat y}(s)=-\varphi(\hat y(s))$ with \emph{initial}
condition $\hat y(0)=y^T$.
\end{proposition}
Clearly \eqref{eq:hamilton-general}
or \eqref{eq:second-order-eq} are not an autonomous system
and hence we cannot apply directly Proposition~\ref{prop:time-reversal}
nonetheless, we can still  consider the following
modification of \eqref{eq:local_HJ-second-form}
\begin{equation}\label{eq:local_HJ-second-form-flipped}
\begin{cases}
c_i^{-1}\dot x_i= -x_i +\sigma\Bigl(\sum_{j\in\pa(i)} w_{ij}x_j\Bigr)\\
\dot w_{ij}=-p^{ij}_\allweights/(mc\phi)\\
\dot p^i_x=-c_i p_x^i+\sum_{k\in\ch(i)} c_k 
\sigma'\Bigl(\sum_{j\in\pa(k)} w_{kj}x_j\Bigr)p_x^k w_{ki}
+ c L_{\xi_i}(x,t)\phi\\
\dot p^{ij}_\allweights(t)=
 c_i p^i_x \sigma'\Bigl(\sum_{m\in\pa(i)} w_{im}x_m\Bigr)x_j
+ c k V_{\u_{ij}}(\allweights,t)\phi
\end{cases}
\end{equation}
which are obtained from \eqref{eq:local_HJ-second-form} by
changing the sign to $\dot p_x$ and $\dot p_\allweights$.
Recalling the definition of the rescaled costates in
\eqref{eq:rescaled-costates} we can cast, in the same spirit
of Proposition~\ref{prop:second-order-dyn} a system of equations without
$p_\allweights$. In particular we have as a corollary of
Proposition~\ref{prop:second-order-dyn} that
\begin{corollary}
The ODE system in \eqref{eq:local_HJ-second-form-flipped} is equivalent to
\begin{equation}\label{eq:second-order-eq-flipped}
\left\{\begin{aligned}
&c_i^{-1}\dot x_i= -x_i +\sigma(a_i);\\
&\ddot w_{ij}=-\frac{\dot \phi}{\phi} \dot w_{ij}
-\frac{c_i}{mc}\lambda^i_x x_j
-\frac{k}{m} V_{\u_{ij}}(\allweights,t);\\
&\dot \lambda^i_x=
\biggl[-\frac{\dot\phi}{\phi}+\frac{d}{dt}\log(\sigma'(a_i))
-c_i\biggr]\lambda^i_x +\sigma'(a_i)\sum_{k\in\ch(i)}c_k\lambda^k_x w_{ki}
+ c L_{\xi_i}(x,t)\sigma'(a_i),
\end{aligned}\right.
\end{equation}
\end{corollary}
\begin{proof}
Let us consider \eqref{eq:second-order-eq}.
The change of sign of $\dot p_\allweights$ only affect the signs of
$\lambda^i_x x_j$ and $ V_{\u_{ij}}(\allweights,t)$ in the $\ddot w_{ij}$
equation, while the change of sign of $\dot p_x$ result in a sign change
of the term $c_i\lambda^i_x$,
$\sigma'(a_i)\sum_{k\in\ch(i)}c_k\lambda^k_x w_{ki}$ and
$L_{\xi_i}(x,t)\sigma'(a_i)$ in the equation for $\dot \lambda^x_i$.
\end{proof}

Equation~\ref{eq:second-order-eq-flipped} is indeed
particularly interesting because it offers an interpretation of the
dynamics of the weights $w$ that is in the spirit of a gradient-base
optimization method. In particular this allow us the extend the
result that we gave in Corollary~\ref{cor:delta-error} to
a full statement on the resulting optimization method 
\begin{proposition}[GD with momentum]\label{prop:GD-with-momentum}
Let $c_i$ be the same for all $i=1,\dots, n$ so that now $c_i=c$, 
and let $\phi(t)=\exp(\theta t)$ with $\theta>0$
then the formal limit of the system in \eqref{eq:second-order-eq-flipped}
as $c\to\infty$
is
\begin{equation}
\begin{cases}
x_i = \sigma(a_i);\\
\ddot w_{ij}=-\theta \dot w_{ij}
-\frac{1}{m}\lambda^i_x x_j
-(k/m) V_{\u_{ij}}(\allweights,t);\\
\lambda^i_x= \sigma'(a_i)\sum_{k\in\ch(i)}\lambda^k_x w_{ki}
+ L_{\xi_i}(x,t)\sigma'(a_i).
\end{cases}
\end{equation}
\end{proposition}

\begin{remark}
This result shows that at least in the case of infinite speed of propagation
of the signal across the network ($c\to\infty$) the dynamics of the weights
prescribed by Hamilton's equation with the costate dynamics that is
reversed (the sign of $\dot p_x$ and $\dot p_\allweights$ is changed)
results in a gradient flow dynamic (heavy-ball
dynamics) that it is interpretable as a gradient descent with momentum
in the discrete. This is true since the term $\lambda^i_x x_j$ in this limit
is exactly the Backprop factorization of the gradient of the term
$L$ with respect to the weights.
\end{remark}

In view of this remark we can therefore conjecture that also for $c$ fixed:
\begin{conjecture}
Equation~\ref{eq:second-order-eq-flipped}
is  a \emph{local optimization scheme} for the loss term $\ell$.
\end{conjecture}
Such result would enable us to use \eqref{eq:second-order-eq-flipped}
with initial Cauchy conditions as desired.

\subsection{Continuous Time Reversal of State and Costate}
Now we show that another possible approach to  the problem 
of solving Hamilton's equation with Cauchy's conditions is to perform 
\emph{simultaneous time-reversal} of both state and costate equation. 
Since in this case the sign flip involves both the Hamiltonian equations 
the approach is referred to as \emph{Hamiltonian Sign Flip} (HSF).
In order to introduce the idea let us begin with the following example.

\begin{example}[LQ control]\label{ex:lq-control}
Let us consider a linear quadratic scalar problem where the functional
in~\eqref{eq:funtional-G} is 
$G(v)=\int_0^T q x^2/2+rv^2/2\, dt$ and $\dot x=ax+bv$ with
$q$, $r$ positive and $a$ and $b$ real parameters.
The associated Hamilton's equations in this case are
\begin{equation}\label{eq:lq-problem}
\dot x=ax-sp,\quad \dot p=-qx-ap,\end{equation}
where $s\equiv -b^2/r$. These equation can be solved with the ansatz
$p(t)=\theta(t) x(t)$, where $\theta$ is some unknown parameter.
Differentiating this expression with respect to time we obtain
\begin{equation}\label{eq:thetadot}
\dot \theta=(\dot p-\theta\dot x)/x,
\end{equation}
and using the \eqref{eq:lq-problem} into this expression we find
$\dot \theta-s\theta^2-2a\theta -q=0$ which is known as
\emph{Riccati equation}, and since $p(T)=0$, because of
boundary \eqref{eq:boundary} this implies $\theta(T)=0$.
Again if instead we try to solve this equation with initial condition
we end up with an unstable solution. However $\theta$ solves an
autonomous ODE with final condition, hence by Proposition~\ref{prop:time-reversal}
 we can solve it with $0$ initial conditions as long as we change the 
 sign of $\dot \theta$. Indeed
the equation $\dot \theta+s\theta^2+2a\theta +q=0$
is asymptotically stable and returns the correct solution of the
Riccati algebraic equation. Now the crucial observation is
that, as we can see from \eqref{eq:thetadot}, the  sign flip
of $\dot \theta$ is equivalent to the \emph{simultaneous}  sign flip of
$\dot x $ and $\dot p$.
\end{example}
In Example~\ref{ex:lq-control}, as we observe 
from \eqref{eq:thetadot}, the  sign flip
of $\dot \theta$ is equivalent to the \emph{simultaneous}  sign flip of
$\dot x $ and $\dot p$.
Inspired by the fact, let us associate the general 
Hamilton's equation~(\eqref{eq:hamilton-general}),
to this system the Cauchy problem 
\begin{equation}\label{eq:hamilton-track}
\begin{pmatrix}
\dot x(t)\\\dot\allweights(t)\\\dot p_x(t)\\\dot p_\allweights(t)
\end{pmatrix}= s(t)
\begin{pmatrix}
f(x(t),\allweights(t),t)\\-p_\allweights(t)/(mc\phi(t))\\
-p_x(t)\cdot f_\xi(x(t),\allweights(t),t)-
 c\ell_\xi(\allweights(t),x(t),t)\phi(t)\\
 -p_x(t)\cdot f_\u(x(t),\allweights(t),t)-
 c\ell_\u(\allweights(t),x(t),t)\phi(t)
\end{pmatrix}
\end{equation}
where for all $t\in[0,T]$, $s(t)\in\{0,1\}$. 
Here we propose two different strategies that extends the sign flip 
discussed for the LQ problem. 

\paragraph{Hamiltonian Track} The basic idea is enforce system stabilization
by choosing $s(t)$ to bound both the Hamiltonian variables.  
This leads to define an
\emph{Hamiltonian track}:
\begin{definition}
Let $S(\xi,\u,p,q)\subset(\R^{n-d}\times\R^N)^2$ for every
$(\xi,\u,p,q)\in (\R^{n-d}\times\R^N)^2$ be a bounded connected set and let
$t\mapsto X(t)$ any
continuous trajectory in the space $(\R^{n-d}\times\R^N)^2$, then we
refer to
\[
\{(t, S(X(t)): t\in[0,T]\}\in[0,T]\times (\R^{n-d}\times\R^N)^2
\]
as \emph{Hamiltonian track (HT)}.
\end{definition}

Then we define $s(t)$ as follow
\begin{equation}
s(t)=\begin{cases}
1& \hbox{if $(x(t),\allweights(t), p_x(t),p_\allweights(t))\in
S((x(t),\allweights(t), p_x(t),p_\allweights(t)))$} \\
-1& \hbox{otherwise}
\end{cases}.
\end{equation}
For instance if we choose $S(\xi,\u,p,q)=\{
(\xi,\u,p,q): |\xi|^2+|\u|^2+|p|^2+|q|^2\le R\}$ we are 
constraining the dynamics of \eqref{eq:hamilton-track} to be bounded
since each time the trajectory
$t\mapsto (x(t),\allweights(t), p_x(t),p_\allweights(t))$ moves outside of
a ball of radius $R$ we are reversing the dynamics by enforcing stability.

\paragraph{Hamiltonian Sign Flip Strategy and time reversal}
We can easily see that the sign flip driven by the policy of 
enforcing the system dynamics into the HT corresponds with 
time reversal of the trajectory, which can nicely be interpreted as 
focus of attention mechanism. 
A simple approximation of the movement into the HT is that of
selecting $s(t) = {\rm sign}(\cos(\bar{\omega} t))$, where $\bar{\omega} = 2\pi \bar{f}$
is an appropriate \textit{flipping frequency} which governs the movement into the HT. 
In the discrete setting of computation the strategy consists of flipping the right-side of Hamiltonian equations sign with a given period. 
In the extreme case the sign flip takes place at any Euler discretization step. 

Here we report the application of the \textit{Hamiltonian Sign Flip} strategy
to the classic Linear Quadratic Tracking (LQT) problem by using a recurrent neural network based on a fully-connected digraph.
The purpose of the reported experiments is to validate the HSF policy, which is in fact
of crucial importance in order to exploit the power of the local propagation presented
in the paper, since the proposed policy enables on-line processing.  

\begin{figure}[t]
  \centering
  \begin{minipage}[b]{0.45\textwidth}
    \centering
    \includegraphics[width=\textwidth]{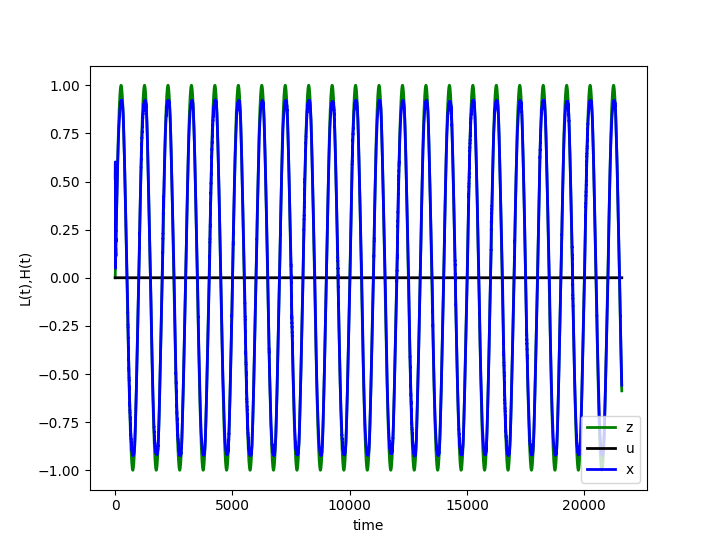}
    \caption{Recurrent net with $5$ neurons, $q=100$ (accuracy term), $r_w=1$ (weight regularization term), $r= 0.1$ (derivative of the weight term).
}
    \label{fig:figure1}
  \end{minipage}
  \hfill
  \begin{minipage}[b]{0.45\textwidth}
    \centering
    \includegraphics[width=\textwidth]{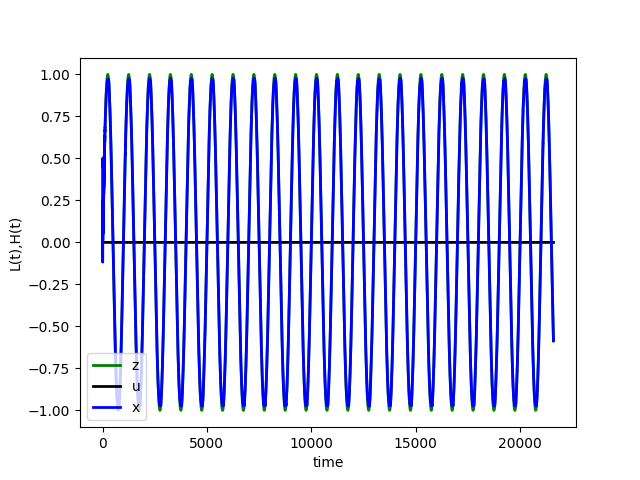}
    \caption{Recurrent net with $5$ neurons, $q=1000$ (accuracy term), $r_w=1$ (weight regularization term), $r= 0.1$ (derivative of the weight term)..}
    \label{fig:figure2}
  \end{minipage}
\end{figure}

The pre-algorithmic framework proposed in the paper, which is based on ODE
can promptly give rise to algorithmic interpretations by numerical solutions. 
In the reported experiments we used Euler's discretization.\\
\noindent \textit{Sinusoidal signals: The effect of the accuracy parameter.}
In this experiment we used a sinusoidal target and a  
recurrent neural network with five neurons, while the objective function was 
$G(v)=\int_0^T q x^2/2+rv^2/2 + r_w^2\, dt$, where we also introduced a regularization
term on the weights. 
The HSF policy gives rise to the expected approximation results. In plot we can
appreciate the effect of the increment of the accuracy term in Fig.~\ref{fig:figure2}.

\noindent \textit{Tracking under hard predictability conditions}
This experiment was conceived to assess the capabilities of the same 
small recurrent neural network with five neurons to track a signal which 
was purposely generated to be quite hard to predict. It is composed of 
patching intervals with cosine functions with constants.
\begin{figure}[t]
  \centering
  \begin{minipage}[b]{0.45\textwidth}
    \centering
    \includegraphics[width=\textwidth]{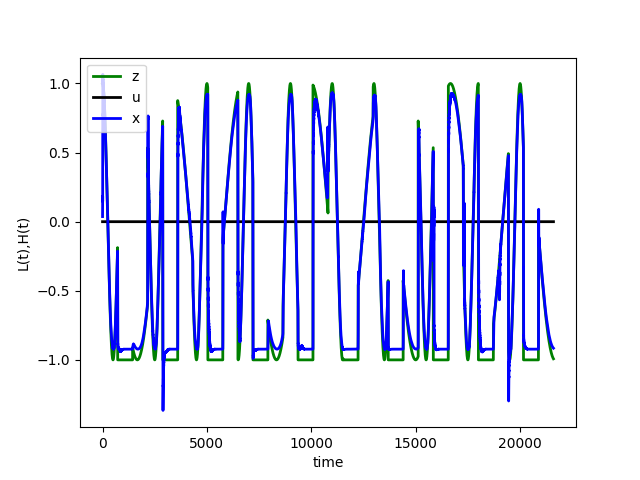}
    \caption{Tracking a highly-predictable signal: number of neurons: $5$,  $q = 100$ (accuracy), weight reg = $1$,  derivative of weight reg = $0.1$}
    \label{fig:figure3}
  \end{minipage}
  \hfill
  \begin{minipage}[b]{0.45\textwidth}
    \centering
    \includegraphics[width=\textwidth]{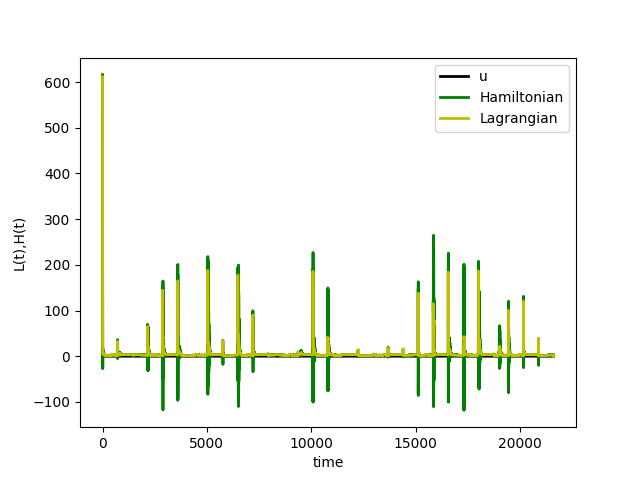}
    \caption{Evolution of the Lagrangian and of the Hamiltonian function for the experiment whose tracking is shown in the left-side figure. }
    \label{fig:figure4}
  \end{minipage}
\end{figure}
The massive experimental analysis on this and related examples confirms effectiveness
of the HSF policy shown in Fig.~\ref{fig:figure3}. The side figure shows the 
behavior of the Lagrangian and of the Hamiltonian term. Interestingly, the last
term gives us insights on the energy exchange with the environment.

\section{Conclusions}
This paper is motivated by the idea of a proposing learning scheme that, like in nature, arises without needing data collections, but simply by on-line processing of the environmental interactions. 
The paper gives two main contributions. 
First, it introduces a local spatiotemporal pre-algorithmic framework that is inspired to classic Hamiltonian equations. It is shown that the corresponding algorithmic formalization leads to  interpret  Backpropagation as a limit case of the proposed diffusion process in case of infinite velocity. This sheds light on the longstanding discussion on the biological plausibility of Backpropagation, since the proposed computational scheme is  local in both space and time. 
This strong result is indissolubly intertwined with a strong limitation. The theory enables such a locality under the assumption that the associated ordinary differential equations are solved as a boundary problem. 
The second result of the paper is that of proposing a method for approximating  the solution of the Hamiltonian problem with boundary conditions by using Cauchy's initial conditions. In particular we show that we can stabilize the learning process by appropriate schemes of time reversal that are related to focus of attention mechanisms. We provide experimental evidence of the effect of the proposed Hamiltonian Sign Flip policy for problems of tracking in automatic control.   

While the proposed local propagation scheme is optimal in the temporal setting and
overcomes the limitations of classic related learning algorithms like BPTT and RTRL,
the given results show that there is no free lunch:  The distinguishing feature
of spatiotemporal locality needs to be sustained by appropriate movement policies 
into the Hamiltonian Track. We expect that other solutions better than 
the HSF policy herein proposed can be developed when dealing with real-word problems. 
This paper must only be regarded as a theoretical contribution which offers
a new pre-algorithmic view of neural propagation. While the provided experiments
support the theory, the application to real-world problems need to activate substantial 
joint research efforts on different application domains.


\bibliography{refs}
\bibliographystyle{iclr2024_conference}

\clearpage
\appendix

\section{Optimal Control}\label{appendix:control}
The classical way in which Hamilton's equations are derived is through
Hamilton-Jacobi-Bellman theorem. So let enunciate this theorem in
a general setting. Here we use the notation $y=(x,\allweights)$ to
stand for the whole state vector and $p=(p_x,p_\allweights)$. We will
also denote with $\alpha$ the control parameters. Moreover to avoid
cumbersome notation in this appendix we will override the notation on the
symbols $n$ and $N$ and we will use them here to denote the dimension
of the state and of the control parameters respectively.

\subsection{Hamilton Jacobi Bellman Theorem}
Consider the classical state model
\begin{equation}\label{eq:state-model}
\dot \state(\time)= \f(\state(\time),\control(\time),\time),\quad t\in(\time_0,T]
\end{equation}
$\f\colon\R^\statedim\times\R^\controldim\times[\time_0,T]\to\R^\statedim$
is a Lipschitz function, $t\mapsto\control(t)$ is the trajectory of the parameters 
of the model, which is assumed to be a \emph{measurable function}
with assigned initial state
$\state^0\in
\R^\statedim$, that is
\begin{equation}\label{eq:state-initialization}
\state(\time_0)=\state^0.
\end{equation}
Let us now pose $\controlspace:=\{\control\colon[\time_0,T]\to\R^\controldim:
\control\hbox{ is measurable}\}$ and given a $\controlbis\in\controlspace$, and given an initial state $\state^0$,
we define the \emph{state trajectory}, that we indicate with
$t\mapsto x(t;\controlbis,\state^0,\time_0)$, the solution of~\eqref{eq:state-model} with
initial condition \eqref{eq:state-initialization}.

Now let us define a cost functional $C$ that we want to minimize:
\begin{equation}\label{eq:cost-functional}
\cost_{\state^0, \time_0}(\control):=
\int_{\time_0}^T \loss(\control(\time), \state(\time;\control,\state^0,\time_0), \time)\, d\time,
\end{equation}
where $\loss(\controlvar,\cdot,\timevar)$ is bounded and Lipshitz $\forall \controlvar\in\R^\controldim$ and
$\forall\timevar\in[\time_0,T]$.
Then the problem
\begin{equation}\label{eq:minimization-of-cost}
\min_{\control\in\controlspace} \cost_{\state^0,\time_0}(\control)
\end{equation}
is  a constrained minimization problem which is usually denoted
as \emph{control problem}~\cite{bardi1997optimal}, assuming that a solution exists.
The first step to address our constrained minimization problem is to define the \emph{value function} or \emph{cost to go}, that is a 
map $\valuef\colon\R^\statedim\times[\time_0,T]\to\R$ defined as 
\[\valuef(\statevar,\timevar):= 
\inf_{\alpha\in\controlspace} C_{\statevar,\timevar}(\alpha),\quad \forall(\statevar,
\timevar)
\in\R^\statedim\times[\time_0,T]\]
and the Hamiltonian function
$H\colon\R^\statedim\times\R^\statedim\times[\time_0,T]\to \R$ as
\begin{equation}\label{eq:hamiltonian}
H(\statevar,\costatevar,\timevar):=\min_{\controlvar\in\R^\controldim} 
\{\costatevar\cdot f(\statevar ,\controlvar,\timevar)+\loss(\controlvar,\statevar,\timevar)\},
\end{equation}
being $\cdot$ the dot product. Then Hamilton-Jacobi-Bellman theorem states
that
\begin{theorem}[Hamilton-Jacobi-Bellman] \label{th:HJB}
Let us assume that $D$ denotes the gradient operator with respect to $\statevar$.
Furthermore, let us assume that $\valuef\in
C^1(\R^\statedim\times[\time_0,T],\R)$ and that the minimum of $C_{\statevar,\timevar}$, Eq.~\eqref{eq:minimization-of-cost}, exists for
every $\statevar\in\R^\statedim$ and for every $\timevar\in[\time_0,T]$. Then
$\valuef$ solves the PDE
\begin{equation}\label{eq:HJB}
v_\timevar(\statevar,\timevar)+H(\statevar,Dv(\statevar,\timevar),\timevar)=0, 
\end{equation}
$(\statevar,\timevar)\in \R^\statedim\times[\time_0,T)$, with terminal condition $\valuef(\statevar,T)=0$, $\forall \statevar\in\R^\statedim$.
Equation~\ref{eq:HJB} is usually referred to as Hamilton-Jacobi-Bellman equation.
\end{theorem}
\begin{proof}
Let $\timevar\in[\time_0,T)$ and $\statevar\in\R^\statedim$.
Furthermore, instead of the optimal
control let us use a constant control $\alpha_1(\time)=\controlvar \in \R^\controldim$ for times $\time\in[\timevar,\timevar+\epsilon]$
and then the optimal control for the remaining temporal interval.
More precisely
let us pose
\[\alpha_2\in\argmin_{\alpha\in\controlspace} C_{\state(\timevar+\eps; \controlvar,\statevar,
\timevar),\timevar+\eps}(\alpha).\]
Now consider the following control
\begin{equation}
\alpha_3(\time)=
\begin{cases}
\alpha_1(\time) &  \text{if} \ \time\in[\timevar,\timevar+\eps)\\
\alpha_2(\time) &  \text{if} \ \time\in[\timevar+\eps,T] \ .
\end{cases}
\end{equation}
Then the cost associated to this control is
\begin{equation}
\begin{split}
C_{\statevar,\timevar}(\alpha_3)&=\int_{\timevar}^{\timevar+\eps} \loss(a, \state(\time;a,
\statevar,\timevar),\time)\,d\time \\
&+\int_{\timevar+\eps}^T \loss(\alpha_2(\time), \state(\time;\alpha_2,\statevar,
\timevar),\time)\,ds \\
&= \int_{\timevar}^{\timevar+\eps} \loss(a, \state(\time;a,
\statevar,\timevar),\time)\,d\time \\
&+ v(\state(\timevar+\eps; a,\statevar,\timevar),\timevar+\eps)
\end{split}
\end{equation}
By definition of value function we also have that $v(\statevar,\timevar)\le
C_{\statevar,\timevar}(\alpha_3)$. When  rearranging this inequality, dividing by
$\eps$, and making use of the above relation we have
\begin{equation}
\begin{split}
& \frac{v(\state(\timevar+\eps; a,\statevar,\timevar),\timevar+\eps)-v(\statevar,\timevar)}{\eps}+ \\
& \frac{1}{\eps}\int_\timevar^{\timevar+\eps} \loss(a, \state(\time;\controlvar,\statevar,\timevar),\time)\,d\time\ge0
\end{split}
\end{equation}
Now taking the limit as $\eps\to 0$ and making use of the fact that
$\state'(\timevar,\controlvar,\statevar,\timevar)=f(\statevar,\controlvar,\timevar)$
we get
\begin{equation}
v_\timevar(\statevar,\timevar)+ Dv(\statevar,\timevar)\cdot 
f(\statevar,\controlvar,\timevar)+\loss(\controlvar,\statevar,\timevar)\ge0.
\end{equation}
Since this inequality holds for any chosen $\controlvar\in\R^\controldim$ we can say that
\begin{equation}
\inf_{\controlvar\in \R^{\controldim}} \{v_\timevar(\statevar,\timevar)+ Dv(\statevar,\timevar)\cdot 
f(\statevar,\controlvar,\timevar)+\loss(\controlvar,\statevar,\timevar)\}\ge 0
\end{equation}
Now we show that the $\inf$ is actually a $\min$ and, moreover, that minimum is
$0$. To do this we simply choose
$\alpha^*\in\argmin_{\alpha\in\controlspace} C_{\statevar,\timevar}(\alpha)$ and denote
$a^*:=\alpha^*(\timevar)$, then 
\begin{equation}
\begin{split}
v(\statevar,\timevar)&= \int_\timevar^{\timevar+\eps} \loss(\alpha^*(\time), \state(\time;
\control^*,\statevar,\timevar),\time)\,d\time \\
&+v(\state(\timevar+\eps;\control^*,\statevar,\timevar).
\end{split}
\end{equation}
Then again dividing by $\eps$ and using that
$\state'(\timevar;\control^*,\statevar,\controlvar)=f(\statevar,\controlvar^*,\timevar)$
we finally get
\begin{equation}
v_\timevar(\statevar,\timevar)+ Dv(\statevar,\timevar)\cdot f(\statevar,
\controlvar^*,\timevar)+\loss(\controlvar^*,\statevar,\timevar)=0
\end{equation}
But since $a^*\in \R^{\controldim}$ and we knew that
$\inf_{\controlvar\in \R^{\controldim}} \{v_\timevar(\statevar,\timevar)+ Dv(\statevar,\timevar)\cdot 
f(\statevar,\controlvar,\timevar)+\loss(\controlvar,\statevar,\timevar)\}\ge 0$
it means that
\begin{align}
\begin{split}
&\inf_{a\in\R^\controldim} \{v_\timevar(\statevar,\timevar)
+ Dv(\statevar,\timevar)\cdot f(\statevar,\controlvar,\timevar)
+\loss(\controlvar,\statevar,\timevar)\}
= \\
&\min_{a\in\R^\controldim} \{v_\timevar(\controlvar,\timevar)+ Dv(\statevar,\timevar)\cdot f(\statevar,\controlvar,\timevar)+\loss(\controlvar,\statevar,\timevar\}=0.
\end{split}
\end{align}
Recalling the definition of $H$ we immediately see that
the last inequality is exactly (HJB). 
\end{proof}

\subsection{Hamilton Equations: The Method of Characteristics}
Now let us define $p(t)=D v(y(t),t)$ so that by definition of the
value function $p(T)=0$ which gives \eqref{eq:boundary}.
Also by differentiating this expression with
respect to time we have
\begin{equation}\label{eq:dotpHJ}
\dot p_k(t)=v_{\xi_k t}(y(t),t)+\sum_{i=1}^n v_{\xi_k\xi_i}(y(t),t)\cdot
\dot y_i.
\end{equation}
Now since $v$ solves \eqref{eq:HJB}, if we differentiate the Hamilton Jacobi equation by $\xi_k$ we obtain:
\[
v_{t\xi_k}(\xi,s)=-H_{\xi_k}(\xi, Dv(\xi,s),s)
-\sum_{i=1}^n H_{\rho_i}(\xi, Dv(\xi,s),s)\cdot v_{\xi_k\xi_i}(\xi,s).
\]
Once we compute this expression on $(y(t),t)$ and we
substitute it back into \eqref{eq:dotpHJ} we get:
\[
\dot p_k(t)=-H_{\xi_k}(y(t), Dv(y(t),t),t)
+\sum_{i=1}^n \Bigl[\dot y_i(t)-H_{\rho_i}(y(t), Dv(y(t),t),t)\Bigr]
\cdot v_{\xi_k\xi_i}(y(t),t).
\]
Now if we choose $y$so that it satisfies $\dot y(t)=H_\rho(y(t), p(t),t)$
the above equation reduces to
\[
\dot p=-H_\xi(y(t),p(t),t).
\]
Applying these equations to the Hamiltonian in \eqref{eq:hamiltonian-general}
we indeed end up with \eqref{eq:hamilton-general}.

\section{Proof of Theorem~\ref{HL-ff-eqs}}\label{append:main-theo}
From  \eqref{RecNeuralNeteq} and the hypothesis on $\ell$ we have that
\[\begin{aligned}
&f^k_{\xi_i}=-c_i\delta_{ik}+c_k\sigma'\Bigl(\sum_{j\in\pa(k)} w_{kj}x_j\Bigr)
\sum_{m\in\pa(k)}w_{km}\delta_{mi}, \qquad
\ell_\xi=L_{\xi_i}(x,t)\\
&f^k_{\u_{ij}}=  c_k \sigma'\Bigl(\sum_{m\in\pa(k)} w_{km}x_m\Bigr)
\sum_{h\in\pa(k)} \delta_{ik}\delta_{jh}x_h,\qquad
\ell_{\u_{ij}}= kV_{\u_{ij}}.
\end{aligned}\]
Then \eqref{eq:hamilton-general} becomes
\begin{equation}\label{eq:local_HJ-first-form}
\begin{cases}
c_i^{-1}\dot x_i= -x_i +\sigma\Bigl(\sum_{j\in\pa(i)} w_{ij}x_j\Bigr)\\
\dot w_{ij}=-p^{ij}_\allweights/(mc\phi)\\
\dot p^i_x=c_i p_x^i-\sum_{k=d+1}^n\sum_{m\in\pa(k)} c_k p_x^k
\sigma'\Bigl(\sum_{j\in\pa(k)} w_{kj}x_j\Bigr) w_{km}\delta_{mi}
- c L_{\xi_i}(x,t)\phi\\
\dot p^{ij}_\allweights(t)=
-\sum_{k=d+1}^n c_k p^k_x \sigma'\Bigl(\sum_{m\in\pa(k)} w_{km}x_m\Bigr)
\sum_{h\in\pa(k)} \delta_{ik}\delta_{jh}x_h
- c k V_{\u_{ij}}(\allweights,t)\phi
\end{cases}
\end{equation}

Now to conclude the proof it is sufficient to apply the
following lemma to conveniently rewrite and switch the sums in
the $\dot p$ equations.

\begin{lemma}\label{lemma:switch-sums}
Let $A$ be the set of the arches of a digraph as in Section~\ref{sec:model},
and let \eqref{eq:all_inputs_are_connected} be true, then
\[A=\{\, (m,k)\in A: k\in\{d+1,\dots,n\}\,\}=\{\,
(m,k)\in A : m\in\{1,\dots,n\}\,\}.\]
Equivalently we may say that $\sum_{k=d+1}^n\sum_{m\in\pa(k)}
=\sum_{m=1}^n\sum_{k\in\ch(m)}$.
\end{lemma}
\begin{proof}
It is an immediate consequences of the fact that the first $d$ neurons
are all parents of some neuron in $\{d+1,\dots, n\}$
(\eqref{eq:all_inputs_are_connected})
and that they do not have themselves any parents
(\eqref{eq:input_do_not_have_parents}).
\end{proof}

\section{Proof of Proposition~\ref{prop:second-order-dyn}}\label{append:gradient}
The equations for $\ddot w_{ij}$ simply follows from differentiation of
the expression for $\dot w_{ij}$ in \eqref{eq:local_HJ-second-form} and the
by the usage of the equation for $\dot p^{ij}_\allweights$ together with
the definition for $\lambda^i_x$ in \eqref{eq:rescaled-costates}.
The equation for $\dot\lambda$ in \eqref{eq:second-order-eq} instead
can be obtained by differentiating with respect to time
\eqref{eq:rescaled-costates} and then using the expression of $\dot p^i_x$
from \eqref{eq:local_HJ-second-form}.

\section{Proof of Proposition~\ref{prop:rnn-stability}}\label{rnn-stability}
Let $\mu(t):=\sigma\Bigl(\sum_{j\in\pa(i)} w_{ij}x_j(t)\Bigr)$ be. 
From the boundedness of $\sigma(\cdot)$ we know that there exists $B>0$ 
such that $|\mu(t)| \leq B$. Now we have 
\begin{align*}
    x_{i}(t) =  x_{i}(0) e^{-\alpha t}
    + \int_{0}^{t} e^{-\alpha(t-\tau)} u(\tau) d\tau
    &\leq x_{i}(0) + B \int_{0}^{t} e^{-\alpha(t-\tau)} d\tau\\
    &\leq x_{i}(0) + \frac{B}{\alpha} (1-e^{-t}) < x_{i}(0) + \frac{B}{\alpha}
\end{align*}

\end{document}